\documentclass[11pt,letterpaper]{article}
\usepackage[T1]{fontenc}
\usepackage{lmodern}
\usepackage[margin=1in]{geometry}
\usepackage{microtype}
\usepackage[hyphens]{url}
\usepackage{graphicx}
\usepackage{tikz}
\usepackage{listofitems}
\usepackage[round]{natbib}
\let\cite\citep

\usepackage{amsmath,amssymb,amsthm,mathtools}
\usepackage{bm}
\usepackage{booktabs}
\usepackage{algorithm}
\usepackage{algorithmic}
\usepackage{needspace}
\usetikzlibrary{arrows.meta,calc}
\usepackage[hidelinks]{hyperref}
\hypersetup{
    pdftitle={Candidate Retention for Abductive Learning},
    pdfauthor={Hao-Yuan He, Yu Liu, and Ming Li}
}

\newtheorem{theorem}{Theorem}
\newtheorem{proposition}{Proposition}

\theoremstyle{definition}
\newtheorem{example}{Example}

\newcommand{\x}{{\bm{x}}}

\renewcommand{\c}{{\bm{c}}}
\newcommand{\s}{{\bm{s}}}
\newcommand{\X}{{\cal X}}
\newcommand{\Y}{{\cal Y}}
\newcommand{\Z}{{\cal Z}}
\newcommand{\KB}{\text{KB}}
\newcommand{\Ind}{\mathbb{I}}
\newcommand{\Exp}{{\mathbb E}}
\newcommand{\Ls}{{\mathcal L}}
\newcommand{\TV}{\mathrm{TV}}
\newcommand{\Unc}{\mathrm{U}}
\newcommand{\AthreeBL}{\ensuremath{\text{A}^{3}\text{BL}}~}
\newcommand*{\img}[1]{%
    \raisebox{-.2\baselineskip}{%
        \includegraphics[
        height=0.9\baselineskip,
        width=0.9\baselineskip,
        keepaspectratio,
        ]{#1}%
    }%
}
\newcommand{\IMGI}{{\img{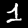}}}
\newcommand{\IMGII}{{\img{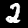}}}

\newlength{\dlen}
\newlength{\stdlen}
\newcommand{\msc}[2]{%
    \setsepchar{.}%
    \readlist*\themean{#1}%
    \readlist*\thestd{#2}%
    \ensuremath{{\mathmakebox[2\dlen][r]{\themean[1]}}.\mathmakebox[2\dlen][l]{\themean[2]}_{\pm \mathmakebox[\stdlen][l]{\thestd[1].\thestd[2]}}}%
}

\title{Candidate Retention for Abductive Learning}
\author{Hao-Yuan He \quad Yu Liu \quad Ming Li\\[0.5em]
    \small School of Artificial Intelligence, Nanjing University\\
    \small National Key Laboratory for Novel Software Technology, Nanjing University\\[0.2em]
    \small\texttt{\{hehy, liuy, lim\}@lamda.nju.edu.cn}}
\date{}

\begin{document}
\maketitle

\begin{abstract}
Abductive learning combines neural perception with symbolic reasoning, using explanations generated by abduction to supervise the perception model. Multiple valid explanations of the same symbolic target can assign conflicting labels to the same inputs. Common policies select a single candidate as a pseudo-label, which may reinforce mistaken assignments, or weight all candidates, which may spread supervision across competing labels. These risks motivate selecting a retained subset to balance supervision sharpness and model-mass coverage. To guide this choice, we bound the coordinate-level supervision error using retained uncertainty, discarded model mass, and model mismatch. For a fixed model and training pair, only the first two terms depend on the retained set. We propose Abductive Candidate Retention (ACR), which uses these terms to guide greedy additions, accepting a candidate when its recovered mass exceeds the increase in retained uncertainty. Experiments show that ACR improves concept accuracy over single-candidate baselines and \AthreeBL in most evaluated aggregated mod-addition settings. Objective ablations support the joint use of uncertainty and posterior mass.
\end{abstract}

\section{Introduction}

Neuro-symbolic learning integrates neural perception with symbolic reasoning to learn from data and background knowledge.
Abductive learning (ABL,~\citealt{zhou2019abductive}) combines neural perception with symbolic reasoning through abduction~\cite{Magnani09Abductive}: a perception model proposes latent concepts from raw inputs, and when these proposals are inconsistent with the observed evidence, abduction generates candidate concept assignments that entail it.
These candidates then supervise the perception model.

However, the background knowledge does not always determine a unique candidate.
Consider \(\mathtt{SUM(\IMGI,\IMGII)=3}\): the addition rule admits \(\mathtt{[\IMGI=0,\IMGII=3]}\), \(\mathtt{[\IMGI=3,\IMGII=0]}\), and other assignments.
Because different candidates label the same images differently, they create \emph{ambiguity} in the concept-level supervision.

\begin{figure}[t]
    \centering
    \includegraphics[width=\textwidth]{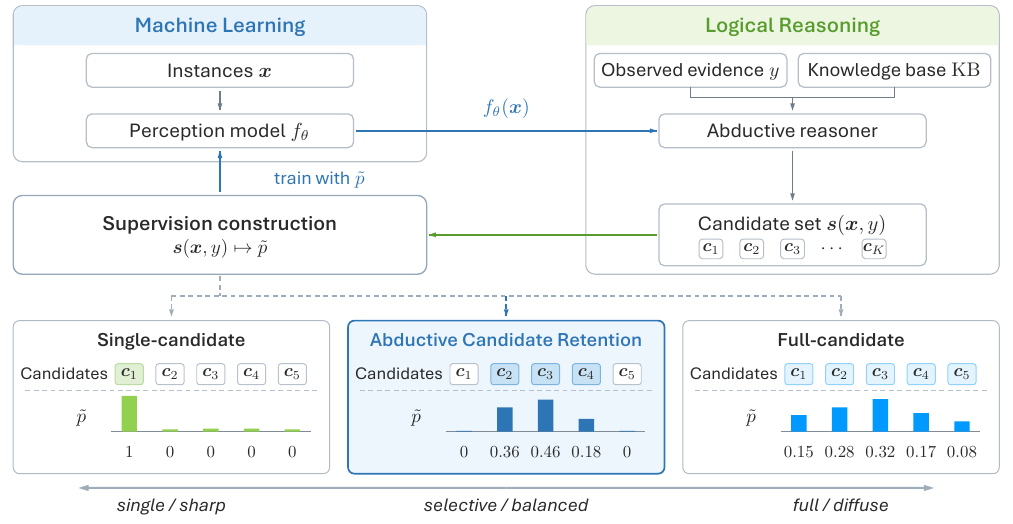}
    \caption{\emph{Candidate retention in the ABL loop.}
        \textbf{Upper:} abduction returns \(\s(\x,y)\); supervision construction marginalizes the retained candidates' weighted labels to train the perception model.
        \textbf{Lower:} three retention policies with schematic candidate weights \(\tilde p\). Shading indicates which candidates are retained.}
    \label{fig:overview}
\end{figure}

Classical ABL selects one candidate as a pseudo-label using the current model~\cite{dai_abl_2019}. \AthreeBL normalizes all valid candidates into a posterior and marginalizes it into soft labels~\cite{he2024a3bl}. Selecting one explanation yields sharp labels but can reinforce a mistaken assignment. Retaining more explanations preserves alternatives, while their disagreements can spread supervision across competing labels. We call the choice of which candidates to include \emph{candidate retention}.

A candidate's contribution to supervision depends on the label support already retained. Two candidates with equal weights can therefore change supervision differently: one may reinforce supported labels, while the other adds competing labels. This combination-dependent effect motivates evaluating each addition through the supervision formed by the retained set.

We propose Abductive Candidate Retention (ACR), which evaluates additions through their recovered model mass and change in coordinate uncertainty (Figure~\ref{fig:overview}). An upper bound on coordinate supervision error contains retained uncertainty, discarded model mass, and model mismatch. For a fixed model and training pair, only the first two depend on the retained set. Let \(M_\theta(R;\x,y)\) denote the retained mass and \(\Unc_\theta(R;\x,y)\) the average coordinate uncertainty. We score a retained set \(R\) by
\[
J_\theta(R;\x,y)
=\Unc_\theta(R;\x,y)+1-M_\theta(R;\x,y).
\]
ACR starts from the highest-weight candidate and adds the candidate that most decreases this criterion, stopping when no addition improves it. Marginalizing the retained weights then supplies concept targets for the neural update.

Our contributions:
\begin{enumerate}
    \item We analyze candidate retention through the concept supervision induced by a subset and give an upper bound on its coordinate-level error.
    \item We develop a greedy rule that evaluates added candidate mass together with changes in coordinate uncertainty, adapting the retained set's composition and size.
    \item We show concept-accuracy gains on most evaluated mod-addition settings, with controlled comparisons and training diagnostics, and evaluate ACR on chess attack and judicial sentencing.
\end{enumerate}

\section{Problem Setup}
\label{sec:preliminaries}

Let \(\X\) be the input space, \(\Z\) a finite concept set, and \(\Y\) the
target space. In typical ABL tasks, a single symbolic target constrains
multiple inputs jointly. Each training example therefore contains an input
sequence \(\x=(x_j)_{j=1}^m\in\X^m\) and an observed target \(y\in\Y\). The concept
\(z_j\in\Z\) for each input \(x_j\) is hidden.

ABL~\cite{zhou2019abductive} combines a perception model
with background knowledge \(\KB\). For each input \(x\), the model assigns a
probability \(f_\theta(k\mid x)\) to every concept \(k\in\Z\).

The knowledge base defines the candidates that entail the observed target:
\begin{equation*}
    \s(\x,y)=\{\c\in\Z^m:\c\land\KB\models y\}.
\end{equation*}
Training reveals \(y\), but it does not identify which \(\c\in\s(\x,y)\)
matches the inputs. ABL uses these valid candidates to learn the hidden
concepts.

\begin{example}
    Consider \texttt{SUM(}\IMGI\texttt{,}\IMGII\texttt{)=3}. The input contains
    two unlabeled digit images, and \texttt{3} is their observed sum. The valid
    candidates are \texttt{[0,3]}, \texttt{[1,2]}, \texttt{[2,1]}, and
    \texttt{[3,0]}. Only one describes the image contents, but its digit labels
    are unavailable during training.
\end{example}

\subsection{Existing ABL Supervision}
Classical ABL selects one valid candidate and uses it as
supervision~\cite{dai_abl_2019}. It commonly ranks candidates by confidence or
Hamming distance. Under the factorized perception model, candidate confidence is
\begin{equation*}
p_\theta(\c\mid\x)=\prod_{j=1}^m f_\theta(c_j\mid x_j).
\end{equation*}
A Hamming score measures the negative distance between \(\c\) and the predicted
concept sequence. Writing either score generically as \(\operatorname{score}_\theta\), ABL selects
\begin{equation*}
\hat{\c}_\theta\in
\arg\max_{\c\in\s(\x,y)}
\operatorname{score}_\theta(\c;\x).
\end{equation*}
The selected candidate provides hard labels for updating \(f_\theta\). The
updated model then reranks the candidates.

\AthreeBL keeps all valid candidates and weights them by
confidence~\cite{he2024a3bl}. Their total confidence is
\begin{equation*}
Z_\theta(\x,y)=\sum_{\c\in\s(\x,y)}p_\theta(\c\mid\x).
\end{equation*}
Normalizing by \(Z_\theta(\x,y)\) gives the candidate posterior
\begin{equation*}
w_\theta(\c\mid\x,y)
=\frac{p_\theta(\c\mid\x)}{Z_\theta(\x,y)},
\qquad \c\in\s(\x,y).
\end{equation*}
We assume \(\s(\x,y)\neq\varnothing\) and \(Z_\theta(\x,y)>0\).

\AthreeBL marginalizes this posterior at each sequence position to construct
soft concept targets for \(f_\theta\).

Classical ABL and \AthreeBL provide the single-candidate and posterior-marginalization reference policies. We next consider how the retained subset shapes concept supervision and use this information to guide selection.

\section{Abductive Candidate Retention}
\label{sec:method}

ACR scores retained sets, converts them into concept targets, and constructs them by greedy selection. Throughout this section, \(w_\theta(\c\mid\x,y)\) denotes a normalized model-induced distribution over the nonempty valid set \(\s(\x,y)\). The product posterior above is one choice; task-specific scores are given in Appendix~\ref{app:experimental_protocol}.

\subsection{A Retention Criterion}
\label{sec:retention-risk}

For a training pair, let \(R\subseteq\s(\x,y)\) be the retained candidate set. Conditioning the candidate distribution on \(R\) gives its retained mass and conditional distribution:

\begin{equation*}
\begin{aligned}
M_\theta(R;\x,y)
&=\sum_{\c\in R}w_\theta(\c\mid\x,y),\\
q_{\theta,R}(\c\mid\x,y)
&=\frac{w_\theta(\c\mid\x,y)}{M_\theta(R;\x,y)}\Ind[\c\in R],
\end{aligned}
\end{equation*}
where \(M_\theta(R;\x,y)>0\).

At position \(j\), the marginal of \(q_{\theta,R}\) is
\begin{equation*}
\pi_{\theta,R,j}(k)
=\Pr_{\c\sim q_{\theta,R}(\cdot\mid\x,y)}(c_j=k).
\end{equation*}

Retained uncertainty averages one minus the largest marginal probability across positions:
\begin{equation*}
\Unc_\theta(R;\x,y)
=\frac1m\sum_{j=1}^m
\left(1-\max_{k\in\Z}\pi_{\theta,R,j}(k)\right).
\end{equation*}
A lower \(\Unc_\theta(R;\x,y)\) gives sharper concept targets. The discarded
posterior mass is \(1-M_\theta(R;\x,y)\).

To relate the retained supervision to concept error, let \(P^\star(\cdot\mid\x,y)\) be the unknown true distribution over valid candidates. It is a point mass when each input sequence has one true candidate.

The retained prediction at position \(j\) is
\begin{equation*}
\hat z_{R,j}(\x,y)\in
\arg\max_{k\in\Z}\pi_{\theta,R,j}(k).
\end{equation*}
Its ideal risk is the expected coordinate error under \(P^\star\):
\begin{equation*}
\mathcal R^\star(R;\x,y)
=\Exp_{\c\sim P^\star(\cdot\mid\x,y)}
\left[\frac1m\sum_{j=1}^m
\Ind[\hat z_{R,j}(\x,y)\neq c_j]\right].
\end{equation*}

Since \(P^\star\) is unknown, we compare the induced supervision with this distribution through the following upper bound.

\begin{theorem}
\label{thm:retained-risk}
For any positive-mass set \(R\subseteq\s(\x,y)\),
\begin{equation*}
\begin{aligned}
\mathcal R^\star(R;\x,y)
&\le \Unc_\theta(R;\x,y)
+\TV\!\left(P^\star,w_\theta\right)\\
&\quad+1-M_\theta(R;\x,y),
\end{aligned}
\end{equation*}
where \(\TV\) denotes total variation distance and both distributions are
conditioned on \((\x,y)\).
\end{theorem}

\begin{proof}
Write \(P=P^\star\), \(w=w_\theta\), \(q_R=q_{\theta,R}\), and
\(M_R=M_\theta(R)\), suppressing \((\x,y)\).
The coordinate loss \(\ell_R(\c)=m^{-1}\sum_j\Ind[\hat z_{R,j}\neq c_j]\)
lies in \([0,1]\), so
\begin{equation*}
\Exp_P[\ell_R]\le \Exp_{q_R}[\ell_R]+\TV(P,q_R).
\end{equation*}
The first term equals \(\Unc_\theta(R)\). On \(R\), \(q_R=w/M_R\ge w\);
outside \(R\), \(q_R=0\).
The absolute differences sum to \(1-M_R\) on each part, so
\(\TV(w,q_R)=1-M_R\). By the triangle inequality,
\begin{equation*}
\begin{aligned}
\TV(P,q_R)&\le \TV(P,w)+\TV(w,q_R)\\
&=\TV(P,w)+1-M_R.
\end{aligned}
\end{equation*}
Combining the two inequalities proves the result.
\end{proof}

\paragraph{Interpretation.}
Under \(q_{\theta,R}\), the coordinate error equals retained uncertainty. Model mismatch and discarded mass bound the change in expected error when moving to \(P^\star\). Thus retained mass describes coverage under the model distribution, and its relation to the true explanation depends on model mismatch. For a fixed model and training pair, the mismatch term is constant across \(R\). ACR uses retained uncertainty and discarded mass, the two subset-dependent terms in the bound, to score candidate additions:
\begin{equation*}
J_\theta(R;\x,y)
=\Unc_\theta(R;\x,y)+1-M_\theta(R;\x,y).
\end{equation*}
Let \(C_K\) contain the \(\min(K,|\s|)\) highest-weight candidates, with fixed tie-breaking. Weights remain normalized over \(\s\), so discarded mass includes candidates outside this search pool. The full valid set globally minimizes \(J_\theta\); the criterion alone does not favor partial retention. ACR uses it to evaluate individual additions along a greedy path from the highest-weight singleton and stops when no addition improves it. This local rule can yield intermediate subsets. Their benefit over the full set is evaluated empirically and is not guaranteed by the bound.

\subsection{Retained Concept Supervision}
\label{sec:retained-supervision}

At round \(t\), the current parameters \(\theta_t\) and retained set \(R_t\) define soft targets \(\pi_{t,j}(k)=\pi_{\theta_t,R_t,j}(k)\) and mass \(M_t=M_{\theta_t}(R_t;\x,y)\). Holding these quantities fixed during the neural update gives the normalized cross-entropy loss
\begin{equation*}
\Ls_{\mathrm{ret}}(\theta;\pi_t)
=\sum_{j=1}^m\sum_{k\in\Z}
\pi_{t,j}(k)[-\log f_\theta(k\mid x_j)].
\end{equation*}

Marginalization pools label support across joint explanations. Several candidates can agree at one position while assigning different labels elsewhere; their weights add to the same target entry at that position. Every input in a sequence receives a target derived from the same retained distribution. A singleton produces one-hot targets, while retaining the full valid set recovers its full-distribution marginals. Intermediate sets preserve the alternatives that contribute to the selected supervision.

In aggregated mod addition, the implementation uses the original retained weights to form the targets:
\begin{equation*}
\begin{aligned}
a_{t,j}(k)
&=\sum_{\c\in R_t}w_{\theta_t}(\c\mid\x,y)\Ind[c_j=k]\\
&=M_t\pi_{t,j}(k).
\end{aligned}
\end{equation*}
Its loss is \(M_t\Ls_{\mathrm{ret}}(\theta;\pi_t)\), with \(M_t\) also fixed during the update. The conditional marginals specify the relative label distribution, and retained mass weights the training example. Appendix~\ref{app:experimental_protocol} gives the scoring and weighting conventions. The standard EM connection for fixed sets, product-posterior weights, and normalized targets is described in Appendix~\ref{app:restricted-em}.

\subsection{Greedy Candidate Retention}
\label{sec:greedy-retention}

We construct \(R\subseteq C_K\) by forward selection, starting with the highest-weight candidate. The set expands only when adding a candidate lowers the criterion.

To see how a candidate changes supervision, write \(M=M_\theta(R;\x,y)\), \(w_{\c}=w_\theta(\c\mid\x,y)\), and \(\pi_j=\pi_{\theta,R,j}\). Adding \(\c\notin R\) yields marginals \(\pi_j^+=\pi_{\theta,R\cup\{\c\},j}\) with
\[
\pi_j^+(k)=\frac{M\pi_j(k)+w_{\c}\Ind[c_j=k]}{M+w_{\c}}.
\]
The candidate contributes mass to its assigned label at each position. The effect on uncertainty depends on where this support falls relative to the existing label masses. Two candidates with equal weights can therefore produce different changes in uncertainty. A weight-only top-\(k\) rule ranks them independently of the current retained set. ACR evaluates their effects on the supervision already formed by that set, using the following change in the criterion.

For \(\c\in C_K\setminus R\), define
\begin{equation*}
\begin{aligned}
\Delta_{\c}J_\theta(R)
&\equiv J_\theta(R\cup\{\c\};\x,y)-J_\theta(R;\x,y)\\
&=\Unc_\theta(R\cup\{\c\};\x,y)-\Unc_\theta(R;\x,y)\\
&\quad-w_\theta(\c\mid\x,y).
\end{aligned}
\end{equation*}
The first difference measures the change in retained ambiguity. The final term
rewards the posterior mass recovered by \(\c\).

At each step, we select the candidate with the smallest
\(\Delta_{\c}J_\theta(R)\). We add it if this value is negative and stop
otherwise.

\begin{algorithm}[h]
\caption{Greedy Candidate Retention}
\label{alg:greedy_retention}
\begin{algorithmic}[1]
\REQUIRE Pool \(C_K\) and full-distribution weights \(w_\theta(\c\mid\x,y)\)
\ENSURE Retained set \(R\)
\STATE Choose \(\c_0\in\arg\max_{\c\in C_K}w_\theta(\c\mid\x,y)\).
\STATE Set \(R\leftarrow\{\c_0\}\).
\WHILE{\(R\neq C_K\)}
    \STATE Choose
    \[
    \c^\dagger\in
    \arg\min_{\c\in C_K\setminus R}\Delta_{\c}J_\theta(R).
    \]
    \IF{\(\Delta_{\c^\dagger}J_\theta(R)\geq 0\)}
        \STATE \textbf{break}
    \ENDIF
    \STATE Update \(R\leftarrow R\cup\{\c^\dagger\}\).
\ENDWHILE
\RETURN \(R\).
\end{algorithmic}
\end{algorithm}

Each accepted step strictly lowers \(J_\theta\) and adds one candidate, so the algorithm terminates at a one-addition local optimum over \(C_K\). With maintained coordinate label masses, selection takes \(O(N^2m)\) time in the worst case, where \(N=|C_K|\le K\). At round \(t\), the selected \(R_t\) defines the fixed targets in Section~\ref{sec:retained-supervision}. Candidate weights and retention are recomputed after the neural update.

\begin{table}[t]
\centering
\caption{Character accuracy on digit addition (\%; mean $\pm$ std).
N/A: did not finish within 24 hours.}
\label{tab:addition}
\footnotesize
\setlength{\tabcolsep}{2.5pt}
\renewcommand{\arraystretch}{1.16}
\settowidth{\stdlen}{\(\scriptstyle 00.00\)}
\renewcommand{\msc}[2]{\ensuremath{#1_{\pm\mathmakebox[\stdlen][l]{#2}}}}
\begin{tabular*}{\textwidth}{@{\extracolsep{\fill}}lcccccccc@{}}
\toprule
& \multicolumn{4}{c}{MNIST} & \multicolumn{4}{c}{KMNIST} \\
\cmidrule(lr){2-5}\cmidrule(lr){6-9}
Method & \(n=1\) & \(n=2\) & \(n=3\) & \(n=4\)
& \(n=1\) & \(n=2\) & \(n=3\) & \(n=4\) \\
\midrule
DeepStochLog
& \msc{99.00}{0.13} & \msc{98.88}{0.02} & \msc{98.94}{0.15} & \msc{98.98}{0.13}
& \msc{94.00}{0.73} & \msc{94.06}{0.30} & \msc{94.00}{0.29} & \msc{93.93}{0.49} \\
NeurASP
& \msc{50.01}{0.44} & \msc{10.21}{0.39} & \msc{10.21}{0.39} & N/A
& \msc{10.01}{0.58} & \msc{10.05}{0.51} & N/A & N/A \\
DeepProbLog
& \msc{97.42}{0.29} & N/A & N/A & N/A
& \msc{80.54}{0.21} & N/A & N/A & N/A \\
ABL-hamming
& \msc{98.49}{0.22} & \msc{98.65}{0.08} & \msc{98.63}{0.12} & \msc{60.41}{40.27}
& \msc{91.48}{0.54} & \msc{92.43}{0.65} & \msc{63.51}{36.91} & \msc{47.74}{37.13} \\
ABL-confidence
& \msc{79.31}{41.35} & \msc{79.32}{41.61} & \msc{77.10}{43.53} & \msc{79.22}{41.65}
& \msc{86.32}{25.40} & \msc{79.80}{34.91} & \msc{72.76}{38.75} & \msc{54.84}{42.97} \\
\AthreeBL
& \msc{98.92}{0.15} & \msc{99.03}{0.09} & \msc{98.68}{0.17} & \msc{98.42}{0.08}
& \msc{94.84}{0.25} & \msc{94.31}{0.21} & \msc{93.16}{1.00} & \msc{92.97}{0.44} \\
ACR (ours)
& \msc{98.95}{0.04} & \msc{99.06}{0.18} & \msc{99.02}{0.07} & \msc{98.93}{0.08}
& \msc{94.44}{0.28} & \msc{95.12}{0.41} & \msc{94.44}{0.18} & \msc{94.35}{0.50} \\
\bottomrule
\end{tabular*}
\end{table}

\begin{figure}[t]
    \centering
    \includegraphics[width=1\textwidth]{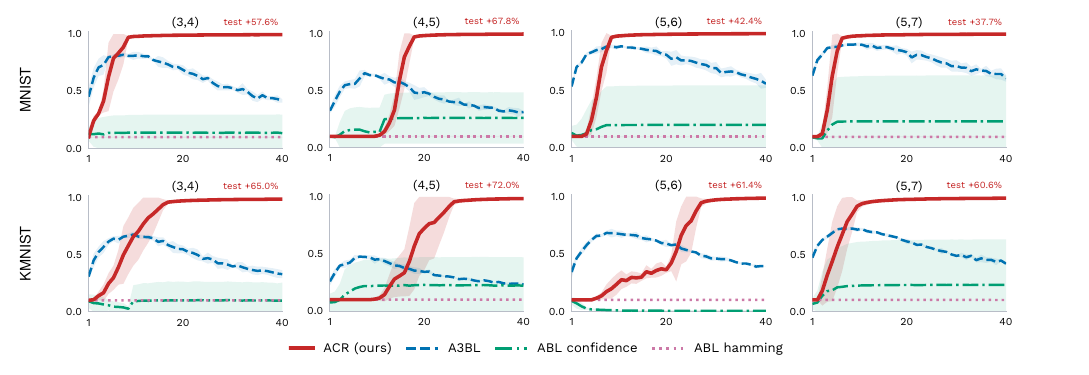}
    \caption{Training curves for four representative modulus pairs per dataset.
    Annotations show ACR's accuracy gain in percentage points over the strongest baseline.
    Full results in Appendix~\ref{app:aggregated_results}.}
    \label{fig:selected_aggregated_trajectories}
\end{figure}

\section{Experiments}
\label{sec:experiments}
We empirically evaluate ACR through two questions:

\begin{itemize}
    \item How does candidate retention compare with the
    single-candidate and posterior-weighted policies when valid candidates disagree on
    hidden concepts?

    \item What roles do the two terms in the objective play, and
    how does the retained set change during training?
\end{itemize}

Section~\ref{sec:exp_effectiveness} compares retention policies on digit addition~\cite{manhaeve_deepproblog_2018} and aggregated
mod addition~\cite{he2025AnalysisOnNeSy}.
Section~\ref{sec:exp_retention} examines the two objective terms through ablations and
retention dynamics. Section~\ref{sec:exp_scope} examines applications to chess attack~\cite{dai_abl_2019} and judicial sentencing~\cite{huang2020ssabl}.

\subsection{Learning from Ambiguous Explanations}
\label{sec:exp_effectiveness}

We begin with two arithmetic settings where the same observed target admits multiple valid digit assignments.

In digit addition, each
example contains two sequences of \(n\) digit images whose hidden labels form two
integers; training observes only their sum~\cite{manhaeve_deepproblog_2018}.
For \(n=1\), \texttt{SUM(}\IMGI\texttt{,}\IMGII\texttt{)=3} reveals the sum
\texttt{3}, while the image labels remain hidden.
MNIST~\cite{lecun1998gradient} and KMNIST~\cite{kmnist} supply the images for
\(n\in\{1,2,3,4\}\). We report character accuracy on the recovered digit
labels, averaged over five runs. As \(n\) grows, more hidden decisions must agree with the same target.

For these arithmetic tasks, we compare ACR with the single-candidate
ABL-hamming and ABL-confidence~\cite{dai_abl_2019}, and
\AthreeBL~\cite{he2024a3bl}. Digit addition also includes DeepProbLog, NeurASP,
and DeepStochLog~\cite{manhaeve_deepproblog_2018,neurasp,deepstochlog}.
For aggregated mod addition, the \AthreeBL runs use a top-32 cap and temperature \(0.2\), retaining all candidates when at most 32 are available.

Both ACR and \AthreeBL maintain high character accuracy across operand lengths, whereas the hard-selection methods have lower means or larger variation across runs (Table~\ref{tab:addition}). On KMNIST, ACR exceeds \AthreeBL for \(n\geq2\), with gains of 0.81--1.38 percentage points; \AthreeBL performs better at \(n=1\).

\begin{figure}[t]
    \centering
    \includegraphics[width=0.96\linewidth]{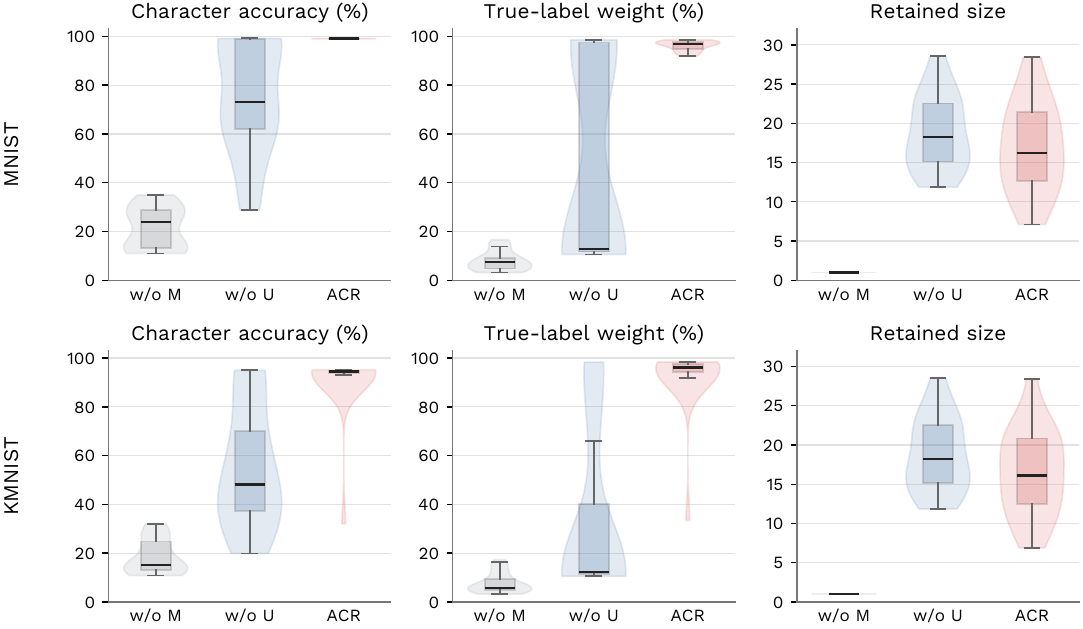}
    \caption{Task-wise ablation distributions across 19 modulus pairs per dataset.
    Columns: character accuracy, TLM, and retained set size.
    ACR w/o mass and ACR w/o uncertainty (w/o M and w/o U in the panels) remove the mass and uncertainty terms.}
    \label{fig:objective_ablation_distributions}
\end{figure}

\begingroup\clubpenalty=10000
Digit addition tests scaling in the number of hidden positions but uses a
single rule per example. Aggregated mod addition introduces a harder coupling:
a shared digit predictor must satisfy two modular rules simultaneously.
For a modulus \(k\), each example
contains two one-digit images with hidden labels \(z_1,z_2\) and one observed
residue:
\[
y_k=(z_1+z_2)\bmod k.
\]
\par\endgroup
Each training example samples \(k\in\{k_1,k_2\}\) and all examples share one
digit predictor, so \texttt{MODSUM\_3(}\IMGI\texttt{,}\IMGII\texttt{)=0}
and \texttt{MODSUM\_4(}\IMGI\texttt{,}\IMGII\texttt{)=3} follow different
modular rules within the same training set. We evaluate 19 modulus pairs with
30,000 examples and three runs per pair~\cite{he2025AnalysisOnNeSy}.

The representative curves in Figure~\ref{fig:selected_aggregated_trajectories} reveal a difference in learning dynamics: ACR maintains high accuracy after the initial rise, while \AthreeBL declines after an earlier peak in some settings. Across the evaluated modulus pairs, ACR improves concept accuracy over \AthreeBL in most settings. The gains vary across pairs; Appendix~\ref{app:aggregated_results} reports the complete results, including small or negative gains.

\subsection{Understanding Candidate Retention}
\label{sec:exp_retention}

The retention objective \(J_\theta\) balances retained uncertainty against discarded posterior mass.
We study these roles through controlled ablations of the two terms and a training trace of the resulting supervision.
Both use retained true-label weight (TLM) as a diagnostic. Let
\(z^\star=(z_j^\star)_{j=1}^m\) denote the ground-truth concepts; we define
\begin{equation*}
\mathrm{TLM}_\theta(R;\x,y)
=\frac{1}{m}\sum_{j=1}^{m}\sum_{\c\in R}
\widetilde w_\theta(\c\mid\x,y)\Ind[c_j=z_j^\star].
\end{equation*}
Here \(\widetilde w_\theta\) is the candidate distribution used in the experiments (Appendix~\ref{app:experimental_protocol}). TLM averages the retained weight assigned to the correct concept without renormalizing over \(R\); it therefore reflects both retained mass and label quality.

The ablations in Table~\ref{tab:objective_ablation} compare three
retention policies under a shared training protocol. Without the mass term,
the initial singleton already minimizes uncertainty, \(\Unc_\theta=0\),
so the strict-improvement rule accepts no additions. Without uncertainty,
each positive-weight addition reduces discarded mass, and the rule retains
the full search pool of \(\min(32,|\s|)\) candidates. Both controls use the
same scoring and raw-weight supervision as ACR.

Averaged over the aggregated mod-addition settings, ACR reaches 94.98\% concept accuracy, compared with 20.01\% for ACR w/o mass and 65.22\% for ACR w/o uncertainty (Table~\ref{tab:objective_ablation}). Removing either term also lowers retained true-label weight. ACR retains fewer candidates on average than ACR w/o uncertainty while assigning more weight to correct labels. Figure~\ref{fig:objective_ablation_distributions} shows how these differences vary across modulus pairs.

\begin{table}[t]
\centering
\caption{Objective ablation averaged over all aggregated mod-addition settings (19 modulus pairs $\times$ 2 datasets).
ACR achieves the highest accuracy and TLM; removing either term degrades both.}
\label{tab:objective_ablation}
\small
\renewcommand{\arraystretch}{1.12}
\begin{tabular}{lrrr}
\toprule
Method & Acc.\,(\%) & TLM\,(\%) & Size \\
\midrule
ACR w/o mass & 20.01 & 7.88 & 1.00 \\
ACR w/o uncertainty & 65.22 & 40.15 & 18.76 \\
ACR & \textbf{94.98} & \textbf{94.25} & 16.68 \\
\bottomrule
\end{tabular}
\end{table}

\begin{figure}[tbp]
    \centering
    \includegraphics[width=0.65\linewidth]{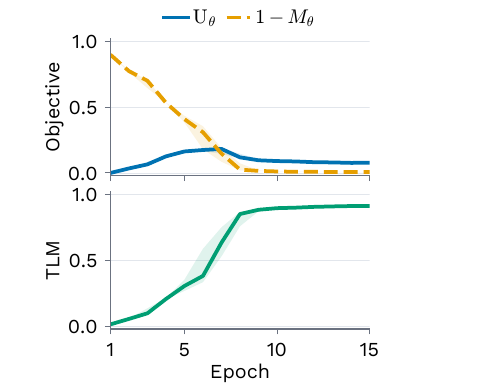}
    \caption{Training dynamics on MNIST mod addition \((3,4)\).
    Top: retained uncertainty \(\Unc_\theta\) and discarded mass \(1-M_\theta\);
    bottom: TLM. Lines: medians; bands: interquartile ranges.}
    \label{fig:retention_dynamics}
\end{figure}

The representative trace in Figure~\ref{fig:retention_dynamics} follows the retained supervision on MNIST mod addition \((3,4)\). Retained uncertainty rises early and then falls, while discarded mass approaches zero. In this trace, sharper coordinate targets coexist later in training with coverage of nearly all model-posterior mass.

\subsection{Beyond Arithmetic Reasoning}
\label{sec:exp_scope}

We next test ACR on two non-arithmetic tasks: chess attack and judicial
sentencing, comparing with ABL-hamming, ABL-confidence~\cite{dai_abl_2019},
and \AthreeBL~\cite{he2024a3bl}.
In chess attack, MNIST images occupy board squares and map
to six piece types~\cite{dai_abl_2019}. Training provides square
locations and whether any pair attacks; piece labels remain hidden. We report
attack and piece accuracy from a single run on about 60,000 training and 10,000 test boards. All
methods share the split and 20-epoch schedule; ACR uses five warm-up epochs of \AthreeBL supervision before switching to retention-based selection.

Judicial sentencing takes judgment facts as input and predicts discrete sentencing
factors, which a knowledge base maps to the observed prison term~\cite{huang2020ssabl}.
We follow the task setup of Huang et al. and report Macro-F1 and Micro-F1, averaged over five runs.

\begin{table}[t]
\centering
\small
\setlength{\tabcolsep}{2.6pt}
\renewcommand{\arraystretch}{1.12}
\caption{Results beyond arithmetic. Best in bold.}
\label{tab:beyond_arithmetic}
\begin{tabular}{lrrrr}
\toprule
& \multicolumn{2}{c}{Chess (\%)} & \multicolumn{2}{c}{Judicial} \\
\cmidrule(lr){2-3}\cmidrule(lr){4-5}
Method & Attack & Piece & Macro-F1 & Micro-F1 \\
\midrule
ABL-hamming & 69.01 & 17.95 & 0.32470 & 0.70256 \\
ABL-confidence & 69.01 & 17.95 & 0.32470 & 0.70256 \\
\AthreeBL & 97.53 & 87.19 & 0.83224 & 0.89574 \\
ACR (ours) & \textbf{99.31} & \textbf{98.29} & \textbf{0.83309} & \textbf{0.89870} \\
\bottomrule
\end{tabular}
\end{table}

\begingroup\clubpenalty=10000
Warm-started ACR improves more on piece identification than on attack prediction (Table~\ref{tab:beyond_arithmetic}). Both marginalization-based methods outperform hard selection, while the judicial task leaves less room for adaptive retention to differentiate from \AthreeBL.
\par\endgroup

\section{Related Work}
\label{sec:related-work}

In this section, we review the two lines of work most relevant to our method: neuro-symbolic learning
and abductive learning.

\Needspace{4\baselineskip}
\subsection{Neuro-Symbolic Learning}

Neuro-symbolic methods connect perception and reasoning through architectures, loss functions, or explicit inference. Early systems embedded declarative rules in connectionist architectures~\cite{kbANN94,nsfa2002}. Logic-as-loss methods translate symbolic constraints into differentiable objectives~\cite{xu_semantic_2018,roychowdhury2021regularizing}; their behavior depends on the chosen relaxation~\cite{vanKrieken2022Analyzing,he2024rill}.

Hybrid systems retain an explicit symbolic reasoner. DeepProbLog extends ProbLog with neural predicates~\cite{manhaeve_deepproblog_2018,deepproblog_aij}, NeurASP couples neural perception with answer-set programming~\cite{neurasp}, and DeepStochLog uses stochastic logic programs~\cite{deepstochlog}. They share a modular connection between neural perception and symbolic reasoning, with different mechanisms for exchanging information. In ABL, abduction returns explanations that become labels for the neural update. This feedback interface makes the candidate-to-supervision step explicit: several valid joint assignments can supply different labels for the same inputs. ACR studies which of these explanations should contribute to the returned supervision and how their labels combine into concept targets.

\subsection{Abductive Learning}

ABL turns abduction, inferring hypotheses that explain observations under
background knowledge~\cite{Magnani09Abductive}, into a learning mechanism
whose central aim is to make perception and reasoning mutually
beneficial~\cite{zhou2019abductive,dai_abl_2019}. The neural model predicts
concepts, abduction revises them under the knowledge base, and the revised
candidates supervise the next neural update.
When several candidates explain the same target, classical ABL selects the one
closest to the current prediction, using Hamming distance or model
confidence. This decision depends on model quality: at cold start, an
unreliable prediction can select a poor pseudo-label, which subsequent updates
may reinforce~\cite{tao2024abl,he2024a3bl}. ABLSim mitigates this dependence
with inter-sample similarity that guides consistency optimization before the
classifier becomes reliable~\cite{huang2021ablsim}.

\begingroup\widowpenalty=10000
Subsequent work extends the loop along three axes. (i)~\emph{Knowledge scope}: GABL enables abduction from inexact ground
propositions~\cite{cai2021gabl}, and related methods learn from raw data,
semi-supervised signals, or inaccurate
rules~\cite{dai2021ablRawData,huang2020ssabl,yang2024safeabl}. The framework
has also been applied to PDE discovery~\cite{ABLPDE2025Gao} and document
recognition~\cite{KESAR2024Gao}.
(ii)~\emph{Identifiability}: formal analyses connect learning reliability to
knowledge-base structure~\cite{tao2024abl,he2025AnalysisOnNeSy}.
(iii)~\emph{Reasoning control}: C-ABL partitions the knowledge base into
sub-bases introduced progressively~\cite{hu2025curriculumabl}.
\par\endgroup

These advances improve what knowledge is available or how reasoning is
structured, but leave a separate design choice: how much of the candidate
posterior should return to the neural model as supervision. This
choice matters because target-level consistency need not identify the intended
latent concepts: different candidates can explain one target while assigning
different concept labels, producing reasoning
shortcuts~\cite{marconato23ReasoningShortcuts,bortolotti2024RSbench}.
\AthreeBL addresses this ambiguity by weighting every valid candidate under the
current model and marginalizing their coordinate labels~\cite{he2024a3bl},
avoiding an early single-candidate commitment. Classical ABL and \AthreeBL
therefore define the single-candidate and marginalization endpoints.
ACR develops the
intermediate range through a greedy subset-selection criterion combining supervision
sharpness and posterior-mass coverage.

\section{Limitations and Future Directions}

ACR operates on the candidate set returned by abduction. Enumeration can become the bottleneck when this space grows combinatorially. The criterion uses unit coefficients, and task-dependent weighting remains unexplored. Our comparisons use fixed \(K\) and temperature; validation-selected settings and retained-size-matched controls would further characterize the benefit of the selection rule. The chess result also calls for repeated runs with matched warm-up schedules. Candidate retention alone does not establish identifiability of the intended concept semantics, which is central to reasoning shortcuts~\cite{marconato23ReasoningShortcuts}.

Future directions include jointly generating and retaining explanations, learning the relative weighting of uncertainty and mass, and extending retention to inaccurate rules or partial labels.

\section{Conclusion}

Candidate retention shapes concept supervision through the labels that explanations support jointly. ACR evaluates each addition by the model mass it recovers and the change it makes to coordinate uncertainty. Our risk bound separates these two selection-dependent quantities from model mismatch; the greedy rule uses them to construct a retained set for each training pair. Experiments show concept-accuracy gains over single-candidate baselines and \AthreeBL in most evaluated aggregated mod-addition settings. Objective ablations support the complementary roles of mass and uncertainty in retaining alternatives while keeping concept targets informative.

\bibliography{references}
\clearpage

\appendix
\section*{Appendix}

\section{Proofs}
\label{app:proofs}

\subsection{Fixed-Support EM Connection}
\label{app:restricted-em}
This subsection considers the product-posterior choice
\(w_\theta(\c\mid\x,y)=p_\theta(\c\mid\x)/Z_\theta(\x,y)\)
and normalized retained targets. For a fixed positive-mass set \(R\), define
\[
\log Z_{\theta,R}(\x,y)=\log\sum_{\c\in R}p_\theta(\c\mid\x).
\]
The following identity is standard in truncated variational EM~\citep{lucke2019truncated}.

\begin{proposition}
\label{prop:restricted-em}
With product-posterior weights and \(R\) fixed during an update, computing
\(q_t=q_{\theta_t,R}\) performs the exact E-step for
\(\log Z_{\theta,R}(\x,y)\). Holding \(q_t\) fixed, an update that reduces
\[
\Ls_{\mathrm{ret}}(\theta;q_t)
=-\Exp_{\c\sim q_t}[\log p_\theta(\c\mid\x)]
\]
performs a generalized M-step.
\end{proposition}

\paragraph{Proof of Proposition~\ref{prop:restricted-em}.}
\begin{proof}
For any distribution \(q\) supported on \(R\), the standard variational
identity gives
\begin{equation*}
\log Z_{\theta,R}(\x,y)
=\mathcal F_{\x,y}(q,\theta)
+\mathrm{KL}\!\left(q\middle\|q_{\theta,R}\right),
\end{equation*}
where \(\mathcal F_{\x,y}(q,\theta)
=\Exp_{\c\sim q}[\log p_\theta(\c\mid\x)]+H(q)\).
At \(\theta=\theta_t\), setting \(q=q_{\theta_t,R}\) makes the KL term
zero and the bound tight: this is the exact E-step.
With \(q_t\) fixed, \(H(q_t)\) is constant in \(\theta\), so reducing
\(\Ls_{\mathrm{ret}}(\theta;q_t)=-\Exp_{q_t}[\log p_\theta(\c\mid\x)]\)
raises \(\mathcal F\) and thus the likelihood bound, giving a generalized
M-step.
\end{proof}

The likelihood statement concerns a fixed support and these normalized
product-posterior targets. The task-specific score and raw-weight convention
used for aggregated mod addition are specified below.

\section{Experimental Protocol}
\label{app:experimental_protocol}

\paragraph{Hardware.}
All experiments were run on a server with two NVIDIA A100-PCIE GPUs
(40\,GB each) and two Intel Xeon Platinum 8358 CPUs.

\paragraph{Repeated runs.}
Digit addition reports mean and standard deviation over five
runs. Aggregated mod addition, its
ablations, and retention dynamics each use three
runs; Figure~\ref{fig:retention_dynamics} shows medians and
interquartile ranges. Chess attack uses one run with a 20-epoch schedule;
ACR uses five warm-up epochs of \AthreeBL supervision. Judicial sentencing
averages over five runs.
\paragraph{Comparison methods.}
Comparison methods use their official implementations, including the
task-specific data splits and hyperparameters, following the evaluation
protocols described above. For aggregated mod addition, ACR uses \(K=32\) and a
temperature of \(0.2\), the same settings as \AthreeBL. The chess warm-up
schedule is stated alongside its results.

\paragraph{Reproducibility.}
The supplementary materials include data-preprocessing, training, and
analysis code, together with task-specific configurations and instructions
for reproducing all reported results.

\section{Full Aggregated Mod-Addition Results}
\label{app:aggregated_results}

We follow the aggregated mod-addition setup of
\citet{he2025AnalysisOnNeSy} and evaluate all 19 modulus pairs classified as
learnable. Figures~\ref{fig:aggregated_mnist_rq} and
\ref{fig:aggregated_kmnist_rq}
report training trajectories for all pairs. ACR achieves higher accuracy than \AthreeBL on most pairs in both datasets. The largest gap in favor of the baseline occurs at \((2,7)\), where \AthreeBL leads by about three points.

\begin{figure}[htbp]
    \centering
    \includegraphics[width=0.95\linewidth]{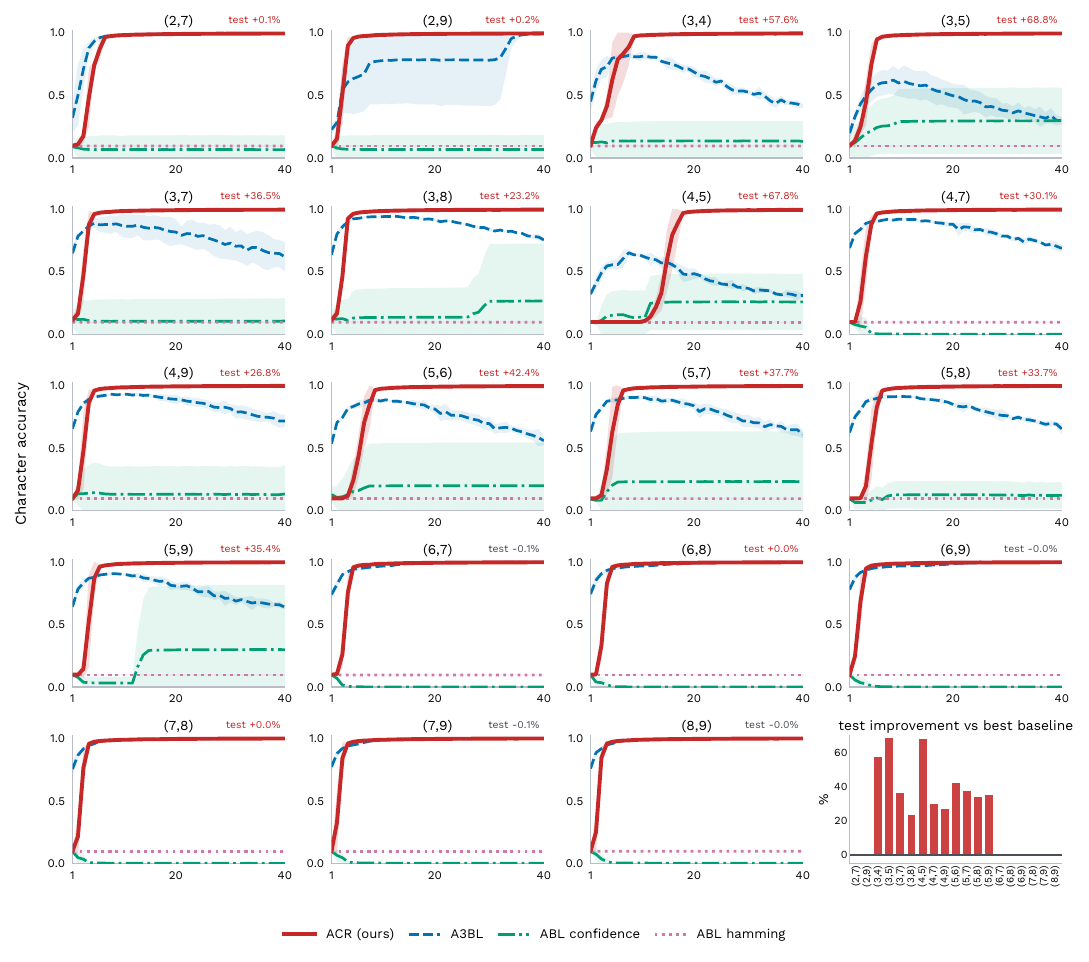}
    \caption{Character-accuracy trajectories for all 19 MNIST modulus pairs. The
    bottom-right panel summarizes ACR's test gain over the strongest baseline.}
    \label{fig:aggregated_mnist_rq}
\end{figure}

\begin{figure}[htbp]
    \centering
    \includegraphics[width=0.95\linewidth]{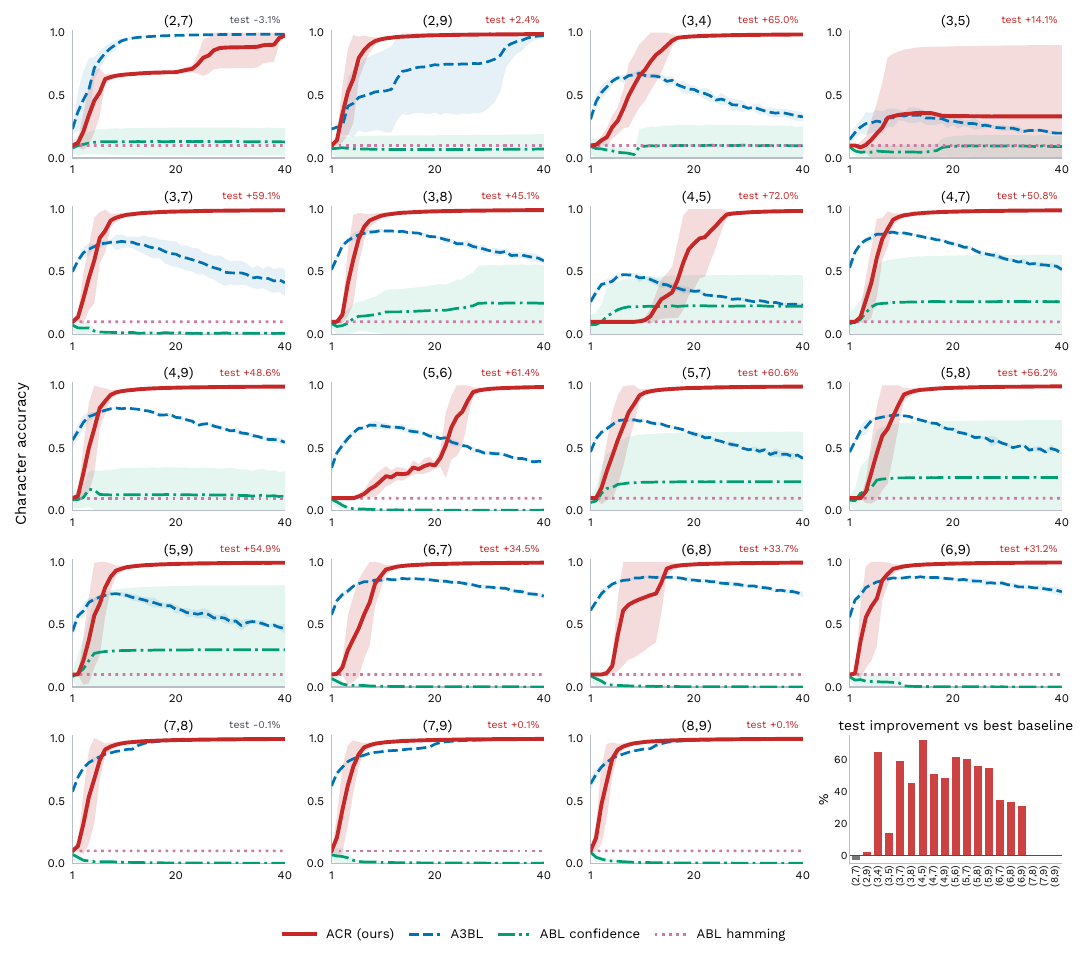}
    \caption{Character-accuracy trajectories for all 19 KMNIST modulus pairs. The
    bottom-right panel summarizes ACR's test gain over the strongest baseline.}
    \label{fig:aggregated_kmnist_rq}
\end{figure}

\end{document}